\documentclass[manuscript,screen,acmsmall]{acmart}

\usepackage{bbding} 
\usepackage[ruled,vlined,shortend,linesnumbered]{algorithm2e} 
\usepackage{enumitem}
\newcommand{\method}{\textsc{Midas}}
\newcommand{\sedanspot}{\textsc{SedanSpot}}
\usepackage{algpseudocode}

\usepackage{numprint} 
\npthousandsep{,} 

\usepackage{hyperref}

\newtheorem{problem}{Problem}
\newtheorem{definition}{Definition}

\newtheorem{theorem}{Theorem}
\newcommand{\codeurl}{https://github.com/Stream-AD/MIDAS}

\AtBeginDocument{%
	\providecommand\BibTeX{{%
		\normalfont B\kern-0.5em{\scshape i\kern-0.25em b}\kern-0.8em\TeX}}}

\setcopyright{acmcopyright}
\acmJournal{TKDD}
\acmYear{2022} \acmVolume{16} \acmNumber{4} \acmArticle{75} \acmMonth{1} \acmPrice{15.00}\acmDOI{10.1145/3494564}

\begin{document}

	\title{Real-Time Anomaly Detection in Edge Streams}


	\author{Siddharth Bhatia}
	\affiliation{%
		\institution{National University of Singapore}
		\country{Singapore}
	}
	\email{siddharth@comp.nus.edu.sg}

	\author{Rui Liu}
	\affiliation{%
		\institution{National University of Singapore}
		\country{Singapore}
	}
	\email{xxliuruiabc@gmail.com}

	\author{Bryan Hooi}
	\affiliation{%
		\institution{National University of Singapore}
		\country{Singapore}
	}
	\email{bhooi@comp.nus.edu.sg}

	\author{Minji Yoon}
	\affiliation{%
		\institution{Carnegie Mellon University}
		\country{United States}
	}
	\email{minjiy@cs.cmu.edu}

	\author{Kijung Shin}
	\affiliation{%
		\institution{KAIST}
		\country{United States}
	}
	\email{kijungs@kaist.ac.kr}

	\author{Christos Faloutsos}
	\affiliation{%
		\institution{Carnegie Mellon University}
		\country{United States}
	}
	\email{christos@cs.cmu.edu}

	\renewcommand{\shortauthors}{S. Bhatia, et al.}

	\begin{abstract}
		Given a stream of graph edges from a dynamic graph, how can we assign anomaly scores to edges in an online manner, for the purpose of detecting unusual behavior, using constant time and memory? Existing approaches aim to detect \emph{individually surprising} edges.

		In this work, we propose \method, which focuses on detecting \emph{microcluster anomalies}, or suddenly arriving groups of suspiciously similar edges, such as lockstep behavior, including denial of service attacks in network traffic data. We further propose \method-F, to solve the problem by which anomalies are incorporated into the algorithm's internal states, creating a `poisoning' effect that can allow future anomalies to slip through undetected. \method-F introduces two modifications: 1) We modify the anomaly scoring function, aiming to reduce the `poisoning' effect of newly arriving edges; 2) We introduce a conditional merge step, which updates the algorithm's data structures after each time tick, but only if the anomaly score is below a threshold value, also to reduce the `poisoning' effect. Experiments show that \method-F has significantly higher accuracy than \method.

		In general, the algorithms proposed in this work have the following properties: (a) they detects microcluster anomalies while providing theoretical guarantees about the false positive probability; (b) they are online, thus processing each edge in constant time and constant memory, and also processes the data orders-of-magnitude faster than state-of-the-art approaches; (c) they provides up to $62\%$ higher ROC-AUC than state-of-the-art approaches.
	\end{abstract}



	\maketitle

	\section{Introduction}
	Anomaly Detection in graphs is a critical problem for finding suspicious behavior in innumerable systems, such as intrusion detection, fake ratings, and financial fraud. This has been a well-researched problem with majority of the proposed approaches \cite{akoglu2010oddball,chakrabarti2004autopart,hooi2017graph,jiang2016catching,kleinberg1999authoritative,shin2018patterns} focusing on static graphs. However, many real-world graphs are dynamic in nature, and methods based on static connections may miss temporal characteristics of the graphs and anomalies.

	Among the methods focusing on dynamic graphs, most of them have edges aggregated into graph snapshots \cite{eswaran2018spotlight,sun2006beyond,sun2007graphscope,koutra2013deltacon,Sricharan,Gupta}. However, to minimize the effect of malicious activities and start recovery as soon as possible, we need to detect anomalies in real-time or near real-time i.e. to identify whether an incoming edge is anomalous or not, as soon as we receive it. In addition, since the number of vertices can increase as we process the stream of edges, we need an algorithm that uses constant memory in the graph size.

	Moreover, fraudulent or anomalous events in many applications occur in microclusters, suddenly arriving groups of suspiciously similar edges, e.g., denial of service attacks in network traffic data and lockstep behavior.
	However, existing methods that process edge streams in an online manner, including \cite{eswaran2018sedanspot,ranshous2016scalable}, aim to detect individually surprising edges, not microclusters, and can thus miss large amounts of suspicious activity. It is worth noting that in other literature, microcluster may have different meanings~\cite{aggarwal2010on,kranen2011the,bah2019an}, while we specifically refer to a group of sudden arriving edges.

	\begin{figure}[!htb]
		\center{\includegraphics[width=0.7\columnwidth]{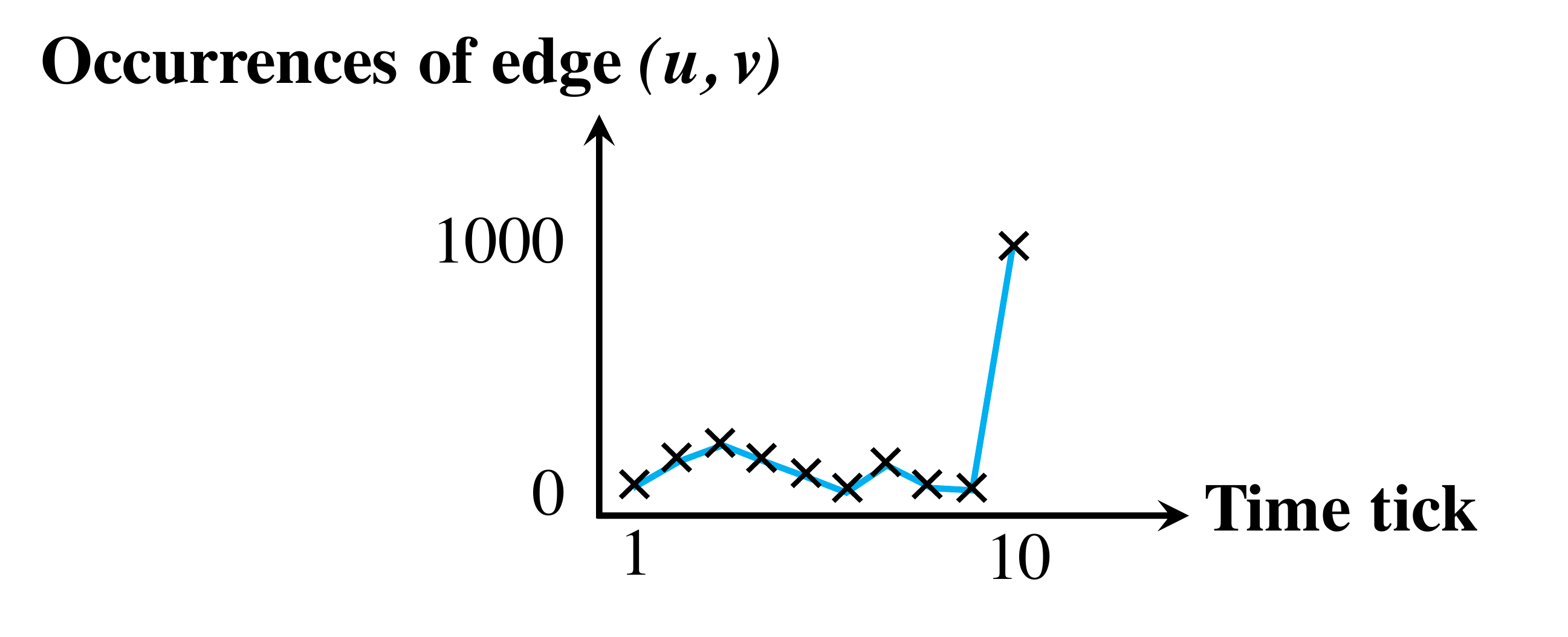}}
		\caption{\label{fig:intro} Time series of a single source-destination pair $(u,v)$, with a large burst of activity at time tick $10$.}
	\end{figure}

	In this work, we propose \method{}, which detect \emph{microcluster anomalies}, or suddenly arriving groups of suspiciously similar edges, in edge streams. Consider the example in Figure \ref{fig:intro} of a single source-destination pair $(u,v)$, which shows a large burst of activity at time $10$. This burst is the simplest example of a microcluster, as it consists of a large group of edges that are very similar to one another. The \method{} algorithm uses count-min sketches (CMS) \cite{cormode2005improved} to count the number of occurrences in each timestamp, then use the chi-squared test to evaluate the degree of deviation and produce a score representing the anomalousness. The higher the score, the more anomalous the edge is. The proposed method uses constant memory and has a constant time complexity processing each edge. Additionally, by using a principled hypothesis testing framework, \method{} provides theoretical bounds on the false positive probability, which those methods do not provide.

	We then propose a relational variant \method-R, which incorporates temporal and spatial relations. In the base version of the \method{} algorithm, the CMS is cleared after every timestamp change. However, some anomalies persist for multiple timestamps. Maintaining partial counts of previous timestamps to the next allows the algorithm to quickly produce a high score when the edge occurs again. This variant also considers the source and destination nodes as additional information that helps determine anomalous edges.

	We also notice that in the base version, edges are merged into the CMS without considering whether the edge is anomalous. However, those anomalous edges will increase the historical mean level of counts, poisoning the scores of future edges. Therefore, we propose the condition merge to prevent this. Also, we simplify the formula of the anomaly score, which grants more efficiency to the detector. Then, we proposed the filtering \method, or \method-F, integrating those changes, and further improves the accuracy of the algorithm.

	Our main contributions are as follows:
	\begin{enumerate}
		\item Streaming Microcluster Detection: We propose a novel streaming approach combining statistical (chi-squared test) and algorithmic (count-min sketch) ideas to detect microcluster anomalies, requiring constant time and memory.
		\item Theoretical Guarantees: In Theorem \ref{thm:bound}, we show guarantees on the false positive probability of \method.
		\item Effectiveness: Our experimental results show that \method{} outperforms baseline approaches by up to $62$\% higher ROC-AUC, and processes the data orders-of-magnitude faster than baseline approaches.
		\item Filtering Anomalies: We propose a variant, \method-F, that introduces two modifications that aim to filter away anomalous edges to prevent them from negatively affecting the algorithm's internal data structures.
	\end{enumerate}

	In the rest of this work, the related works in the area of anomaly detection will be briefly described in section~\ref{sec:RelatedWork}. The problem this work is trying to solve is given in section~\ref{sec:Problem}. The base MIDAS algorithm, the relational variant MIDAS-R and the theoretical guarantee are proposed or proved in section~\ref{sec:MIDAS}. The filtering variant MIDAS-F and its underlying concepts are given in section~\ref{sec:MIDAS-F}. The time and space complexity are analyzed in section~\ref{sec:Complexity}. Experiments and analysis on several real-world datasets are shown in section~\ref{sec:Experiment}.

	Reproducibility: Our code and datasets are publicly available at \url{\codeurl}.

	\section{Related Work}\label{sec:RelatedWork}
	Our work is closely related to areas like graph processing \cite{fang2020survey,yan2020g,semertzidis2019finding,he2015distributed}, streaming algorithms \cite{Bahri2018ASN,Mu2017StreamingCW,khan2018composite,rusu2009sketching,liakos2020rapid}, streaming graph analysis \cite{luo2020dynamic,nakamuramerlin,boniol2020series2graph}, and anomaly detection \cite{DBLP:conf/sdm/BogdanovFMPRS13,saebi2020efficient,noble2003graph,kulkarni2017network,bigdeli2018incremental}. In this section, we limit our review only to previous approaches detecting anomalous signs on static and dynamic graphs. See \cite{akoglu2015graph} for an extensive survey on graph-based anomaly detection.

	\noindent \textbf{Anomaly detection in static graphs} can be classified by which anomalous entities (nodes, edges, subgraph, etc.) are spotted.

	\begin{itemize}
		\item Anomalous node detection:
		\textsc{OddBall} \cite{akoglu2010oddball} extracts egonet-based features and finds empirical patterns with respect to the features.
		Then, it identifies nodes whose egonets deviate from the patterns, including the count of triangles, total weight, and principal eigenvalues.
		\textsc{CatchSync} \cite{jiang2016catching} computes node features, including degree and authoritativeness~\cite{kleinberg1999authoritative}, then spots nodes whose neighbors are notably close in the feature space.
		\item Anomalous subgraph detection:
		FRAUDAR \cite{hooi2017graph} and k-cores \cite{shin2018patterns} measure the anomalousness of nodes and edges, detecting a dense subgraph consisting of many anomalous nodes and edges.
		\item Anomalous edge detection:
		AutoPart \cite{chakrabarti2004autopart} encodes an input graph based on similar connectivity among nodes, then spots edges whose removal reduces the total encoding cost significantly.
		NrMF \cite{tong2011non} factorize the adjacency matrix and flag edges with high reconstruction error as outliers.
	\end{itemize}

	\noindent \textbf{Anomaly detection in graph streams} use as input a series of graph snapshots over time.
	We categorize them similarly according to the type of anomaly detected:

	\begin{itemize}
		\item Anomalous node detection:
		DTA/STA \cite{sun2006beyond} approximates the adjacency matrix of the current snapshot based on incremental matrix factorization, then spots nodes corresponding to rows with high reconstruction error. \cite{aggarwal2011outlier} dynamically partitions the network graph to construct a structural connectivity model and detect outliers in graph streams.
		\item Anomalous subgraph detection:
		Given a graph with timestamps on edges, \textsc{CopyCatch} \cite{beutel2013copycatch} spots near-bipartite cores where each node is connected to others in the same core densly within a short time.
		\item Anomalous event detection:
		\textsc{SpotLight} \cite{eswaran2018spotlight} detects sudden appearance of many unexpected edges, and \textsc{AnomRank} \cite{yoon2019fast} spots sudden changes in 1st and 2nd derivatives of PageRank.
	\end{itemize}

	\noindent \textbf{Anomaly detection in edge streams} use as input a stream of edges over time.
	Categorizing them according to the type of anomaly detected:

	\begin{itemize}
		\item Anomalous node detection:
		Given an edge stream, \textsc{HotSpot} \cite{yu2013anomalous} detects nodes whose egonets suddenly and significantly change.
		\item Anomalous subgraph detection:
		Given an edge stream, \textsc{DenseAlert} \cite{shin2017densealert} identifies dense subtensors created within a short time.
		\item Anomalous edge detection: Only the methods in this category are applicable to our task, as they operate on edge streams and output a score per edge.
		RHSS \cite{ranshous2016scalable} focuses on sparsely-connected parts of a graph but was evaluated in \cite{eswaran2018sedanspot} and was outperformed by \sedanspot \cite{eswaran2018sedanspot}.
		\sedanspot{} uses a customized PageRank to detect edge anomalies based on edge occurrence, preferential attachment, and mutual neighbors in sublinear space and constant time per edge.
		PENminer \cite{belth2020mining} explores the persistence of activity snippets, i.e., the length and regularity of edge-update sequences' reoccurrences.
		F-FADE \cite{chang2021f} aims to detect anomalous interaction patterns by factorizing the frequency of those patterns.
		These methods can effectively detect anomalies, but they require a considerable amount of time.

		We compare with \sedanspot, PENminer, and F-FADE, however, as shown in Table~\ref{tab:comparison}, neither method aims to detect microclusters, or provides guarantees on false positive probability.
	\end{itemize}

	\begin{table}[!htb]
		\centering
		\caption{Comparison of relevant edge stream anomaly detection approaches.}
		\label{tab:comparison}
		\begin{tabular}{@{}lccc|c@{}}
			\toprule
			& \rotatebox{90}{\sedanspot~\cite{eswaran2018sedanspot}}
			& \rotatebox{90}{PENminer~\cite{belth2020mining}}
			& \rotatebox{90}{F-FADE~\cite{chang2021f}}
			& {\bf \rotatebox{90}{\method}} \\ \midrule
			\textbf{Microcluster Detection} & & & & \CheckmarkBold \\
			\textbf{Guarantee on False Positive Probability} & & & & \CheckmarkBold \\
			\textbf{Constant Memory} & \Checkmark & & \Checkmark & \CheckmarkBold \\
			\textbf{Constant Update Time} & \Checkmark & \Checkmark & \Checkmark & \CheckmarkBold \\
			\bottomrule
		\end{tabular}
	\end{table}

	\section{Problem}\label{sec:Problem}

	Let $\mathcal{E} = \{e_1, e_2, \cdots\}$ be a stream of edges from a time-evolving graph $\mathcal{G}$.
	Without loss of generality, we assume the graph is a directed multigraph.
	Each arriving edge is a tuple $e_i = (u_i, v_i, t_i)$ consisting of a source node $u_i \in \mathcal{V}$, a destination node $v_i \in \mathcal{V}$, and a time of occurrence $t_i$, which is the time at which the edge was added to the graph.
	We do not assume the set of vertices $\mathcal{V}$ is known a priori.
	We do not assume the length of the edge stream $\mathcal{E}$ or the node set $\mathcal{V}$ is known as a priori.
	At any moment, the algorithm only knows its internal states and an incoming edge from the stream.

	As the algorithm primarily detects microclusters, we first give the definition of a microcluster.
	\begin{definition}
		Given an edge $e$, a detection period $T\ge 1$, and a threshold $\beta>1$.
		We say there is a microcluster if it satisfies

		\begin{equation}
			\dfrac{c(e,(n+1)T)}{c(e,nT)}>\beta \text{ or } \dfrac{c(e,(n+1)T)}{c(e,nT)}<\dfrac 1 \beta
		\end{equation}
		where $c(e,nT)$ is the occurrence count of $e$ within period $nT$.
	\end{definition}

	A typical value of $T$ is 1, that is, the algorithm detects if there is a burst between adjacent timestamps.
	$\beta$ is a user-defined parameter, it depends on the actual use cases and applications.
	A higher $\beta$ means the system is less sensitive to minor fluctuations.

	Next, we give the formal problem statement.

	\begin{problem}
		Given an evolving edge stream $\mathcal E$, return an anomaly score for each incoming edge $e$ according to the contextual information.
		A higher anomaly score indicates the edge is more suspicious to part of a microcluster, or a burst of edges.
	\end{problem}

	The desired properties of our algorithm are as follows:

	\begin{itemize}
		\item {\bf Microcluster Detection:} It should detect suddenly appearing bursts of activity that share many repeated nodes or edges, which we refer to as microclusters.
		\item {\bf Guarantees on False Positive Probability:} Given any user-specified probability level $\epsilon$ (e.g. $1\%$), the algorithm should be adjustable so as to provide a false positive probability of at most $\epsilon$ (e.g. by adjusting a threshold that depends on $\epsilon$). Moreover, while guarantees on the false positive probability rely on assumptions about the data distribution, we aim to make our assumptions as weak as possible.
		\item {\bf Constant Memory and Update Time:} For scalability in the streaming setting, the algorithm should run in constant memory and constant update time per newly arriving edge. Thus, its memory usage and update time should not grow with the length of the stream or the number of nodes in the graph.
	\end{itemize}

	\section{\method{} and \method-R Algorithms}\label{sec:MIDAS}

	\subsection{Overview}

	Next, we describe our \method{} and \method-R approaches. The following provides an overview:

	\begin{enumerate}
		\item {\bf Streaming Hypothesis Testing Approach:} We describe our \method{} algorithm, which uses streaming data structures within a hypothesis testing-based framework, allowing us to obtain guarantees on false positive probability.
		\item {\bf Detection and Guarantees:} We describe our decision procedure for determining whether a point is anomalous, and our guarantees on false positive probability.
		\item {\bf Incorporating Relations:} We extend our approach to the \method-R algorithm, which incorporates relationships between edges temporally and spatially\footnote{We use `spatially' in a graph sense, i.e. connecting nearby nodes, not to refer to any other continuous spatial dimension.}.
	\end{enumerate}

	\subsection{\method: Streaming Hypothesis Testing Approach}

	\subsubsection{Streaming Data Structures}

	In an offline setting, there are many time-series methods that could detect such bursts of activity. However, in an online setting, recall that we want memory usage to be bounded, so we cannot keep track of even a single such time series. Moreover, there are many such source-destination pairs, and the set of sources and destinations is not fixed a priori.

	To circumvent these problems, we maintain two types of count-min sketch (CMS)~\cite{cormode2005improved} data structures. Assume we are at a particular fixed time tick $t$ in the stream; we treat time as a discrete variable for simplicity. Let $s_{uv}$ be the total number of edges from $u$ to $v$ up to the current time. Then, we use a single CMS data structure to approximately maintain all such counts $s_{uv}$ (for all edges $uv$) in constant memory: at any time, we can query the data structure to obtain an approximate count $\hat{s}_{uv}$.

	Secondly, let $a_{uv}$ be the number of edges from $u$ to $v$ in the current time tick (but not including past time ticks). We keep track of $a_{uv}$ using a similar CMS data structure, the only difference being that we reset this CMS data structure every time we transition to the next time tick. Hence, this CMS data structure provides approximate counts $\hat{a}_{uv}$ for the number of edges from $u$ to $v$ in the current time tick $t$.

	\subsubsection{Hypothesis Testing Framework}

	Given approximate counts $\hat{s}_{uv}$ and $\hat{a}_{uv}$, how can we detect microclusters? Moreover, how can we do this in a principled framework that allows for theoretical guarantees?

	Fix a particular source and destination pair of nodes, $(u,v)$, as in Figure \ref{fig:intro}. One approach would be to assume that the time series in Figure \ref{fig:intro} follows a particular generative model: for example, a Gaussian distribution. We could then find the mean and standard deviation of this Gaussian distribution. Then, at time $t$, we could compute the Gaussian likelihood of the number of edge occurrences in the current time tick, and declare an anomaly if this likelihood is below a specified threshold.

	However, this requires a restrictive Gaussian assumption, which can lead to excessive false positives or negatives if the data follows a very different distribution. Instead, we use a weaker assumption: that the mean level (i.e. the average rate at which edges appear) in the current time tick (e.g. $t=10$) is the same as the mean level before the current time tick $(t<10)$. Note that this avoids assuming any particular distribution for each time tick, and also avoids a strict assumption of stationarity over time.

	Hence, we can divide the past edges into two classes: the current time tick $(t=10)$ and all past time ticks $(t<10)$. Recalling our previous notation, the number of events at $(t=10)$ is $a_{uv}$, while the number of edges in past time ticks $(t<10)$ is $s_{uv} - a_{uv}$.

	Under the chi-squared goodness-of-fit test, the chi-squared statistic is defined as the sum over categories of $\frac{(\text{observed} - \text{expected})^2}{\text{expected}}$. In this case, our categories are $t=10$ and $t<10$. Under our mean level assumption, since we have $s_{uv}$ total edges (for this source-destination pair), the expected number at $t=10$ is $\frac{s_{uv}}{t}$, and the expected number for $t<10$ is the remaining, i.e. $\frac{t-1}{t} s_{uv}$. Thus the chi-squared statistic is:

	\begin{align*}
		X^2 &= \frac{(\text{observed}_{(t=10)} - \text{expected}_{(t=10)})^2}{\text{expected}_{(t=10)}} \\
		&+ \frac{(\text{observed}_{(t<10)} - \text{expected}_{(t<10)})^2}{\text{expected}_{(t<10)}}\\
		&= \frac{(a_{uv} - \frac{s_{uv}}{t})^2}{\frac{s_{uv}}{t}} + \frac{((s_{uv} - a_{uv}) - \frac{t-1}{t} s_{uv})^2}{\frac{t-1}{t} s_{uv}}\\
		&= \frac{(a_{uv} - \frac{s_{uv}}{t})^2}{\frac{s_{uv}}{t}} + \frac{(a_{uv} - \frac{s_{uv}}{t})^2}{\frac{t-1}{t} s_{uv}}\\
		&= (a_{uv} - \frac{s_{uv}}{t})^2 \frac{t^2}{s_{uv}(t-1)}
	\end{align*}
	Note that both $a_{uv}$ and $s_{uv}$ can be estimated by our CMS data structures, obtaining approximations $\hat{a}_{uv}$ and $\hat{s}_{uv}$ respectively. This leads to our following anomaly score, using which we can evaluate a newly arriving edge with source-destination pair $(u,v)$:

	\begin{definition}[Anomaly Score]
		Given a newly arriving edge $(u,v,t)$, our anomaly score is computed as:
		\begin{align}
			\text{score}(u,v,t) = (\hat{a}_{uv} - \frac{\hat{s}_{uv}}{t})^2 \frac{t^2}{\hat{s}_{uv}(t-1)}
		\end{align}
	\end{definition}

	Algorithm \ref{alg:midas} summarizes our \method{} algorithm.

	\begin{algorithm}
		\caption{\method:\ Streaming Anomaly Scoring \label{alg:midas}}
		\KwIn{Stream of graph edges over time}
		\KwOut{Anomaly scores per edge}
		{\bf $\triangleright$ Initialize CMS data structures:} \\
		Initialize CMS for total count $s_{uv}$ and current count $a_{uv}$ \\
		\While{new edge $e=(u,v,t)$ is received:}{
				{\bf $\triangleright$ Update Counts:} \\
			Update CMS data structures for the new edge $uv$\\
			{\bf $\triangleright$ Query Counts:} \\
			Retrieve updated counts $\hat{s}_{uv}$ and $\hat{a}_{uv}$\\
			{\bf $\triangleright$ Anomaly Score:}\\
			{\bf output} $\text{score}((u,v,t)) = (\hat{a}_{uv} - \frac{\hat{s}_{uv}}{t})^2 \frac{t^2}{\hat{s}_{uv}(t-1)}$\\
		}
	\end{algorithm}

	\subsection{Detection and Guarantees}

	While Algorithm~\ref{alg:midas} computes an anomaly score for each edge, it does not provide a binary decision for whether an edge is anomalous or not.
	We want a decision procedure that provides binary decisions and a guarantee on the false positive probability: i.e. given a user-defined threshold $\epsilon$, the probability of a false positive should be at most $\epsilon$.
	Intuitively, the key idea is to combine the approximation guarantees of CMS data structures with properties of a chi-squared random variable.

	The key property of CMS data structures we use is that given any $\epsilon$ and $\nu$, for appropriately chosen CMS data structure sizes ($w = \lceil ln \frac{2}{\epsilon} \rceil, b = \lceil \frac{e}{\nu} \rceil$) \cite{cormode2005improved}, with probability at least $1-\frac{\epsilon}{2}$, the estimates $\hat{a}_{uv}$ satisfy:

	\begin{align}
		\hat{a}_{uv} \le a_{uv} + \nu \cdot N_t
	\end{align}

	where $N_t$ is the total number of edges in the CMS for $a_{uv}$ at time tick $t$.
	Since CMS data structures can only overestimate the true counts, we additionally have

	\begin{align}
		s_{uv} \le \hat{s}_{uv}
	\end{align}

	Define an adjusted version of our earlier score:

	\begin{align}
		\tilde{a}_{uv} = \hat{a}_{uv} - \nu N_t
	\end{align}

	To obtain its probabilistic guarantee, our decision procedure computes $\tilde{a_{uv}}$, and uses it to compute an adjusted version of our earlier statistic:

	\begin{align}
		\tilde{X^2} = (\tilde{a}_{uv} - \frac{\hat{s}_{uv}}{t})^2 \frac{t^2}{\hat{s}_{uv}(t-1)}
	\end{align}

	Note that the usage of $X^2$ and $\tilde{X^2}$ are different.
	$X^2$ is used as the score of individual edges while $\tilde{X^2}$ facilitates making binary decisions.

	Then our main guarantee is as follows:

	\begin{theorem}[False Positive Probability Bound]
		\label{thm:bound}
		Let $\chi_{1-\epsilon/2}^2(1)$ be the $1-\epsilon/2$ quantile of a chi-squared random variable with 1 degree of freedom.
		Then:
		\begin{align}
			P(\tilde{X^2} > \chi_{1-\epsilon/2}^2(1)) < \epsilon
		\end{align}
		In other words, using $\tilde{X^2}$ as our test statistic and threshold $\chi_{1-\epsilon/2}^2(1)$ results in a false positive probability of at most $\epsilon$.
	\end{theorem}

	\begin{proof}
		Recall that
		\begin{align}
			X^2 = (a_{uv} - \frac{s_{uv}}{t})^2 \frac{t^2}{s_{uv}(t-1)}
		\end{align}
		was defined so that it has a chi-squared distribution. Thus:
		\begin{align}
			\label{eq:cond1}
			P(X^2 \le \chi_{1-\epsilon/2}^2(1)) = 1-\epsilon/2
		\end{align}
		At the same time, by the CMS guarantees we have:
		\begin{align}
			\label{eq:cond2}
			P(\hat{a}_{uv} \le a_{uv} + \nu \cdot N_t) \ge 1-\epsilon/2
		\end{align}

		By union bound, with probability at least $1-\epsilon$, both these events \eqref{eq:cond1} and \eqref{eq:cond2} hold, in which case:

		\begin{align*}
			\tilde{X^2} &= (\tilde{a}_{uv} - \frac{\hat{s}_{uv}}{t})^2 \frac{t^2}{\hat{s}_{uv}(t-1)}\\
			& = (\hat{a}_{uv} - \nu \cdot N_t - \frac{\hat{s}_{uv}}{t})^2 \frac{t^2}{\hat{s}_{uv}(t-1)}\\
			& \le (a_{uv} - \frac{s_{uv}}{t})^2 \frac{t^2}{s_{uv}(t-1)}\\
			& = X^2 \le \chi_{1-\epsilon/2}^2(1)
		\end{align*}
		Finally, we conclude that
		\begin{align}
			P(\tilde{X^2} > \chi_{1-\epsilon/2}^2(1)) < \epsilon.
		\end{align}
	\end{proof}

	\subsection{Incorporating Relations}

	In this section, we describe our \method-R approach, which considers edges in a {\bf relational} manner: that is, it aims to group together edges that are nearby, either temporally or spatially.

	\textbf{Temporal Relations:} Rather than just counting edges in the same time tick (as we do in \method), we want to allow for some temporal flexibility: i.e. edges in the recent past should also count toward the current time tick, but modified by a reduced weight. A simple and efficient way to do this using our CMS data structures is as follows: at the end of every time tick, rather than resetting our CMS data structures for $a_{uv}$, we scale all its counts by a fixed fraction $\alpha \in (0, 1)$. This allows past edges to count toward the current time tick, with a diminishing weight. Note that we do not consider $0$ or $1$, because $0$ clears all previous values when the time tick changes and hence does not include any temporal effect; and $1$ does not scale the CMS data structures at all.

	\textbf{Spatial Relations:} We would like to catch large groups of spatially nearby edges: e.g. a single source IP address suddenly creating a large number of edges to many destinations, or a small group of nodes suddenly creating an abnormally large number of edges between them. A simple intuition we use is that in either of these two cases, we expect to observe {\bf nodes} with a sudden appearance of a large number of edges. Hence, we can use CMS data structures to keep track of edge counts like before, except counting all edges adjacent to any node $u$. Specifically, we create CMS counters $\hat{a}_u$ and $\hat{s}_u$ to approximate the current and total edge counts adjacent to node $u$. Given each incoming edge $(u,v)$, we can then compute three anomaly scores: one for edge $(u,v)$, as in our previous algorithm; one for source node $u$, and one for destination node $v$. Finally, we combine the three scores by taking their maximum value. Another possibility of aggregating the three scores is to take their sum and we discuss the performance of summing the scores in Section \ref{sec:exp}. Algorithm \ref{alg:midasr} summarizes the resulting \method-R algorithm.

	\begin{algorithm}
		\caption{\method-R:\ Incorporating Relations \label{alg:midasr}}
		\KwIn{Stream of graph edges over time}
		\KwOut{Anomaly scores per edge}
		{\bf $\triangleright$ Initialize CMS data structures:} \\
		Initialize CMS for total count $s_{uv}$ and current count $a_{uv}$ \\
		Initialize CMS for total count $s_{u}, s_{v}$ and current count $a_{u}, a_{v}$ \\
		\While{new edge $e=(u,v,t)$ is received:}{
				{\bf $\triangleright$ Update Counts:} \\
			Update CMS data structures for the new edge $uv$, source node $u$ and destination node $v$\\
			{\bf $\triangleright$ Query Counts:} \\
			Retrieve updated counts $\hat{s}_{uv}$ and $\hat{a}_{uv}$\\
			Retrieve updated counts $\hat{s}_u,\hat{s}_v,\hat{a}_{u},\hat{a}_{v}$\\
			{\bf $\triangleright$ Compute Edge Scores:}\\
			$\text{score}(u,v,t) = (\hat{a}_{uv} - \frac{\hat{s}_{uv}}{t})^2 \frac{t^2}{\hat{s}_{uv}(t-1)}$\\
			{\bf $\triangleright$ Compute Node Scores:}\\
			$\text{score}(u,t) = (\hat{a}_{u} - \frac{\hat{s}_{u}}{t})^2 \frac{t^2}{\hat{s}_{u}(t-1)}$\\
			$\text{score}(v,t) = (\hat{a}_{v} - \frac{\hat{s}_{v}}{t})^2 \frac{t^2}{\hat{s}_{v}(t-1)}$\\
			{\bf $\triangleright$ Final Scores:}\\
			$\textbf{output} \max\{ \text{score}(u,v,t), \text{score}(u,t), \text{score}(v,t) \}$
		}
	\end{algorithm}

	\section{\method-F: Filtering Anomalies}\label{sec:MIDAS-F}

	In \method{} and \method-R, in addition to being assigned an anomaly score, all normal and anomalous edges are also always recorded into the internal CMS data structures, regardless of their score. However, this inclusion of anomalous edges creates a `poisoning' effect which can allow future anomalies to slip through undetected.

	Let us consider a simplified case of a denial of service attack where a large number of edges arrive between two nodes within a short period of time. \method{} and \method-R analysis can be divided into three stages.

	In the first stage, when only a small number of such edges have been processed, the difference between the current count, $\hat{a}_{uv}$, and the expected count, $\frac{\hat{s}_{uv}}{t}$, is relatively small, so the anomaly score is low. This stage will not last long as the anomaly score will increase rapidly with the number of occurrences of anomalous edges.

	In the second stage, once the difference between these two counters becomes significant, the algorithm will return a high anomaly score for those suspicious edges.

	In the third stage, as the attack continues, i.e. anomalous edges continue to arrive, the expected count of the anomalous edge will increase. As a result, the anomaly score will gradually decrease, which can lead to false negatives, i.e. the anomalous edges being considered as normal edges, which is the `poisoning' effect due to the inclusion of anomalies in the CMS data structures.

	Therefore, to prevent these false negatives, we introduce the improved filtering \method{} (\method-F) algorithm. The following provides an overview:

	\begin{enumerate}
		\item {\bf Refined Scoring Function:} The new formula of the anomaly score only considers the information of the current time tick and uses the mean value of the previous time ticks as the expectation.

		\item {\bf Conditional Merge:} The current count $a$ for the source, destination and edge are no longer merged into the total count $s$ immediately. We determine whether they should be merged or not at the end of the time tick conditioned on the anomaly score.
	\end{enumerate}

	\subsection{Refined Scoring Function}

	During a time tick, while new edges continue to arrive, we only assign them a score, but do not directly incorporate them into our CMS data structures as soon as they arrive. This prevents anomalous edges from affecting the subsequent anomaly scores, which can possibly lead to false negatives. To solve this problem, we refine the scoring function to delay incorporating the edges to the end of the current time tick using a conditional merge as discussed in Section \ref{sec:ConditionalMerge}.

	As defined before, let $a_{uv}$ be the number of edges from $u$ to $v$ in the current time tick (but not including past time ticks). But unlike \method{} and \method-R, in \method-F, we define $s_{uv}$ to be the total number of edges from $u$ to $v$ up to the previous time tick, not including the current edge count $a_{uv}$. By not including the current edge count immediately, we prevent a high $a_{uv}$ from being merged into $s_{uv}$ so that the anomaly score for anomalous edges is not reduced.

	In the \method-F algorithm, we still follow the same assumption: that the mean level in the current time tick is the same as the mean level before the current time tick. However, instead of dividing the edges into two classes: past and current time ticks, we only consider the current time ticks. Similar to the chi-squared statistic of \cite{bhatia2020midas}, our statistic is as below.

	\begin{align*}
		X^2&=\frac{(\text{observed}-\text{expected})^2}{\text{expected}} \\
		&=\frac{\displaystyle\left(a_{uv}-\frac{s_{uv}}{t-1}\right)^2}{\displaystyle\frac{s_{uv}}{t-1}} \\
		&=\frac{\displaystyle\left[a_{uv}^2-\frac{2a_{uv}s_{uv}}{t-1}+\left(\frac{s_{uv}}{t-1}\right)^2\right](t-1)}{\displaystyle s_{uv}} \\
		&=\frac{\displaystyle a_{uv}^2(t-1)^2-2a_{uv}s_{uv}(t-1)+s_{uv}^2}{\displaystyle s_{uv}(t-1)} \\
		&=\frac{(a_{uv}+s_{uv}-a_{uv}t)^2}{s_{uv}(t-1)} \\
	\end{align*}

	Both $a_{uv}$ and $s_{uv}$ can be estimated by our CMS data structures, obtaining approximations $\hat{a}_{uv}$ and $\hat{s}_{uv}$ respectively. We will use this new score as the anomaly score for our \method-F algorithm.

	\begin{definition}[\method-F Anomaly Score]
		Given a newly arriving edge $(u,v,t)$, our anomaly score for this edge is computed as:
		\begin{equation}
			\label{eqn:FilteringCore.AnomalyScore}
			score(u,v,t)=\frac{(\hat{a}_{uv}+\hat{s}_{uv}-\hat{a}_{uv}t)^2}{\hat{s}_{uv}(t-1)}
		\end{equation}
	\end{definition}

	\subsection{Conditional Merge}\label{sec:ConditionalMerge}

	At the end of the current time tick, we decide whether to add $a_{uv}$ to $s_{uv}$ or not based on whether the edge $(u,v)$ appears normal or anomalous.

	We introduce $c_{uv}$ to keep track of the anomaly score. Whenever the time tick changes, if $c_{uv}$ is less than the pre-determined threshold $\theta$, then the corresponding $a_{uv}$ will be added to $s_{uv}$; otherwise, the expected count, i.e., $\frac{s_{uv}}{t-1}$ will be added to $s_{uv}$ to keep the mean level unchanged. We add $a_{uv}$ only when the cached score $c_{uv}$ is less than the pre-determined threshold $\theta$ to prevent anomalous instances of $a_{uv}$ from being added to the $s_{uv}$, which would reduce the anomaly score for an anomalous edge in the future time ticks.

	To store the latest anomaly score $c_{uv}$, we use a CMS-like data structure resembling the CMS data structure for $a$ and $s$ used in \method{} and \method-R. The only difference is that the updates to this data structure do not increment the existing occurrence counts, but instead override the previous values. In the remaining part of the paper, we refer to this CMS-like data structure as CMS for convenience.

	To efficiently merge the CMS data structure for $a$ into the CMS data structure for $s$, we need to know which buckets in the same hash functions across the multiple CMS data structures correspond to a particular edge. However, the algorithm does not store the original edges after processing. Therefore it is necessary that for each entity (edge, source, destination), the three CMS data structures for $a$, $s$, $c$ use the same layout and the same hash functions for each hash table so that the corresponding buckets refer to the same edge and we can do a bucket-wise merge. In practice, the nine CMS data structures can be categorized into three groups, corresponding to the edges, source nodes, and destination nodes, respectively. Only the three CMS data structures within the same group need to share the same structure.

	The conditional merge step is described in Algorithm~\ref{alg:FilteringCore.Merge}.

	\begin{algorithm}[!htb]
		\caption{\textsc{Merge}}
		\label{alg:FilteringCore.Merge}
		\KwIn{CMS for $s$, $a$, $c$, threshold $\theta$}
		\For{$\hat{s}$, $\hat{a}$, $\hat{c}$ from CMS buckets}{
			\If{$\hat{c}<\theta$}{
				$\hat{s} = \hat{s}+\hat{a}$ \\
			}\ElseIf{$t\ne 1$}{
				$\hat{s} = \hat{s}+\dfrac{\hat{s}}{t-1}$ // $\hat{s}$ is up-to-date until $t-1$ \\
			}
		}
	\end{algorithm}

	We also incorporate temporal and spatial relations as done in \method-R. For temporal relations, at the end of every time tick, rather than resetting our CMS data structures for $a_{uv}$, we scale all its counts by a fixed fraction $\alpha \in (0, 1)$. This allows past edges to count toward the current time tick, with a diminishing weight. For spatial relations, we use CMS data structures to keep track of the anomaly score of each edge like before, except considering all edges adjacent to any node $u$. Specifically, we create CMS counters $\hat{c}_u$ to keep track of the anomaly score for each node $u$ across all its neighbors. Given each incoming edge $(u,v)$, we can then compute three anomaly scores: one for edge $(u,v)$, as in \method\ and \method-R; one for source node $u$, and one for destination node $v$.

	Algorithm \ref{alg:FilteringCore.Algorithm} summarizes the resulting \method-F algorithm. It can be divided into two parts: 1) regular edge processing in lines $13$ to $24$, where we compute anomaly scores for each incoming edge and update the relevant counts, and 2) scaling and merging steps in lines $6$ to $12$, where at the end of each time tick, we scale the current counts by $\alpha$ and merge them into the total counts.

	\begin{algorithm}
		\caption{\method-F}
		\label{alg:FilteringCore.Algorithm}
		\KwIn{Stream of graph edges over time, threshold $\theta$}
		\KwOut{Anomaly scores per edge}
		{\bf $\triangleright$ Initialize CMS data structures:} \\
		Initialize CMS data structure for total count $s_{uv}$, current count $a_{uv}$, anomaly score $c_{uv}$ \\
		Initialize CMS data structure for total count $s_{u}$, current count $a_{u}$, anomaly score $c_{u}$ \\
		Initialize CMS data structure for total count $s_{v}$, current count $a_{v}$, anomaly score $c_{v}$ \\
		\While{new edge $e=(u,v,t)$ is received}{
			\If(// Time tick changes){$t\ne t_{internal}$}{
					{\bf $\triangleright$ Merge Counts:} \\
				\Call{Merge}{$\hat{s}_{uv}$, $\hat{a}_{uv}$, $\hat{c}_{uv}$, $\theta$} \\
				\Call{Merge}{$\hat{s}_u$, $\hat{a}_u$, $\hat{c}_u$, $\theta$} \\
				\Call{Merge}{$\hat{s}_v$, $\hat{a}_v$, $\hat{c}_v$, $\theta$} \\
				Scale CMS data structures for $a_{uv}$, $a_{u}$, $a_{v}$ by $\alpha$ \\
				$t_{internal} = t$ \\
			}
			{\bf $\triangleright$ Update Counts:} \\
			Update CMS data structure for $a$ for new edge $uv$ and nodes $u,v$ \\
			{\bf $\triangleright$ Query Counts:} \\
			Retrieve updated counts $\hat{s}_{uv}$ and $\hat{a}_{uv}$\\
			Retrieve updated counts $\hat{s}_u,\hat{s}_v,\hat{a}_{u},\hat{a}_{v}$\\
			{\bf $\triangleright$ Compute Scores:}\\
			$c_{uv} = \dfrac{(\hat{a}_{uv}+\hat{s}_{uv}-\hat{a}_{uv}t)^2}{\hat{s}_{uv}(t-1)}$ \\
			$c_u = \dfrac{(\hat{a}_u+\hat{s}_u-\hat{a}_{u}t)^2}{\hat{s}_u(t-1)}$ \\
			$c_v = \dfrac{(\hat{a}_v+\hat{s}_v-\hat{a}_{v}t)^2}{\hat{s}_v(t-1)}$ \\
			Update CMS data structure for $c$ for edge $uv$ and nodes $u,v$\\
			{\bf $\triangleright$ Final Scores:}\\
			$\textbf{output} \max\{ c_{uv},c_u,c_v\}$
		}
	\end{algorithm}

	\section{Time and Memory Complexity}\label{sec:Complexity}

	In terms of memory, \method, \method-R and \method-F only need to maintain the CMS data structures over time, which are proportional to $O(wb)$, where $w$ and $b$ are the number of hash functions and the number of buckets in the CMS data structures; which is bounded with respect to the data size.

	For time complexity, the only relevant steps in Algorithms \ref{alg:midas}, \ref{alg:midasr} and \ref{alg:FilteringCore.Algorithm} are those that either update or query the CMS data structures, which take $O(w)$ (all other operations run in constant time). Thus, time complexity per update step is $O(w)$.

	For \method{}-F, additionally, at the end of each time tick, $a$ is merged into $s$, as shown in Algorithm \ref{alg:FilteringCore.Merge}. At the end of each time tick, the algorithm needs to iterate over all hash functions and buckets. Thus, time complexity per time tick is $O(wb)$.

	\section{Experiments}\label{sec:Experiment}
	\label{sec:exp}

	In this section, we evaluate the performance of \method{}, \method-R and \method-F compared to \sedanspot{} on dynamic graphs. We aim to answer the following questions:

	\begin{enumerate}[label=\textbf{Q\arabic*.}]
		\item {\bf Accuracy:} How accurately does \method{} detect real-world anomalies compared to baselines, as evaluated using the ground truth labels? How will hyperparameters affect the accuracy?
		\item {\bf Scalability:} How does it scale with input stream length? How does the time needed to process each input compare to baseline approaches?
		\item {\bf Real-World Effectiveness:} Does it detect meaningful anomalies in case studies on \emph{Twitter} graphs?
	\end{enumerate}

	\textbf{Datasets:} \emph{DARPA}~\cite{lippmann1999results} is an intrusion detection dataset created in $1998$.
	It has $25K$ nodes, $4.5M$ edges, and $46K$ timestamps.
	The dataset records IP-IP connections from June $1$ to August $1$.
	Due to the relatively sparse time density, we use minutes as timestamps. \emph{CTU-13}~\cite{garcia2014empirical} is a botnet traffic dataset captured in the CTU University in $2011$.
	It consists of botnet samples from thirteen different scenarios.
	We mainly focus on those with denial of service attacks, i.e., scenario $4$, $10$, and $11$.
	The dataset includes $371K$ nodes, $2.5M$ edges, and $33K$ timestamps, where the resolution of timestamps is one second. \emph{UNSW-NB15}~\cite{moustafa2015unsw} is a hybrid of real normal activities and synthetic attack behaviors.
	The dataset contains only $50$ nodes, but has $2.5M$ records and $85K$ timestamps.
	Each timestamp in the dataset represents an interval of one second. \emph{TwitterSecurity}~\cite{rayana2016less} has $2.6M$ tweet samples for four months (May-Aug $2014$) containing Department of Homeland Security keywords related to terrorism or domestic security.
	Entity-entity co-mention temporal graphs are built on a daily basis. Ground truth contains the dates of major world incidents. \emph{TwitterWorldCup}~\cite{rayana2016less} has $1.7M$ tweet samples for the World Cup $2014$ season (June $12$-July $13$).
	The tweets are filtered by popular/official World Cup hashtags, such as \#worldcup, \#fifa, \#brazil, etc.
	Entity-entity co-mention temporal graphs are constructed on one hour sample rate.

	Note that we use different time tick resolutions for different datasets, demonstrating our algorithm is capable of processing datasets with various edge densities.

	\textbf{Baselines:}
	As described in the Related Work, we use \sedanspot \cite{eswaran2018sedanspot}, PENminer \cite{belth2020mining}, and F-FADE \cite{chang2021f} as our baselines.

	\textbf{Evaluation Metrics:}
	All the methods output an anomaly score per edge (higher is more anomalous).
	We report the area under the receiver operating characteristic curve (ROC-AUC, higher is better).

	\subsection{Experimental Setup}

	All experiments are carried out on a $2.4 GHz$ Intel Core $i9$ processor, $32 GB$ RAM, running OS $X$ $10.15.2$.
	We implement our algorithm in C++ and use the open-source implementations of \sedanspot, PENminer, and F-FADE provided by the authors, following parameter settings as suggested in the original papers.

	We use $2$ hash functions for the CMS data structures, and set the number of CMS buckets to $1024$ to result in an approximation error of $\nu=0.003$.
	For \method-R and \method-F, we set the temporal decay factor $\alpha$ as $0.5$.
	For \method-F, the default threshold $\theta$ is $1000$. We discuss the influence of $\alpha$ and the threshold $\theta$ in the following section. Unless otherwise specified, all experiments are repeated 21 times and the median performance (ROC-AUC, running time, etc.) is reported to minimize the influence of randomization in hashing. Also, note that the reported running time does not include I/O.

	\subsection{Accuracy}

	Table~\ref{tab:Experiment.AUROC} shows the ROC-AUC of \sedanspot, PENminer, F-FADE, \method, \method-R, and \method-F on the \emph{DARPA}, \emph{CTU-13}, and \emph{UNSW-NB15} datasets since only these three datasets have ground truth available for each edge. On \emph{DARPA}, compared to the baselines, \method\ algorithms increase the ROC-AUC by $6$\%-$53$\%, on \emph{CTU-13} by $13$\%-$62$\%, and on \emph{UNSW-NB15} by $12$\%-$30$\%.

	\begin{table*}[!htb]
		\centering
		\caption{ROC-AUC (standard deviation)}\label{tab:Experiment.AUROC}
		\begin{tabular}{lrrrlll}
			\toprule
			Dataset & PENminer & F-FADE & \sedanspot & \method & \method-R       & \method-F       \\
			\midrule
			\emph{DARPA}      & 0.8267   & 0.8451 & 0.6442     & 0.9042 (0.0032) & 0.9514 (0.0012)         & \textbf{0.9873} (0.0009) \\
			\emph{CTU-13}     & 0.6041   & 0.8028 & 0.6397     & 0.9079 (0.0049) & 0.9703 (0.0009)         & \textbf{0.9843} (0.0004) \\
			\emph{UNSW-NB15}  & 0.7028   & 0.6858 & 0.7575     & 0.8843 (0.0079) & \textbf{0.8952} (0.0028) & 0.8517 (0.0013)         \\
			\bottomrule
		\end{tabular}
	\end{table*}

	Figures \ref{fig:AUCdarpa}, \ref{fig:AUCctu}, and \ref{fig:AUCunsw}  plot the ROC-AUC vs. running time for the baselines and our methods on the \emph{DARPA}, \emph{CTU-13}, and \emph{UNSW-NB15} datasets respectively. Note that \method, \method-R, and \method-F achieve a much higher ROC-AUC compared to the baselines, while also running significantly faster.

	\begin{figure}[!ht]
		\center
		\includegraphics[width=0.8\columnwidth]{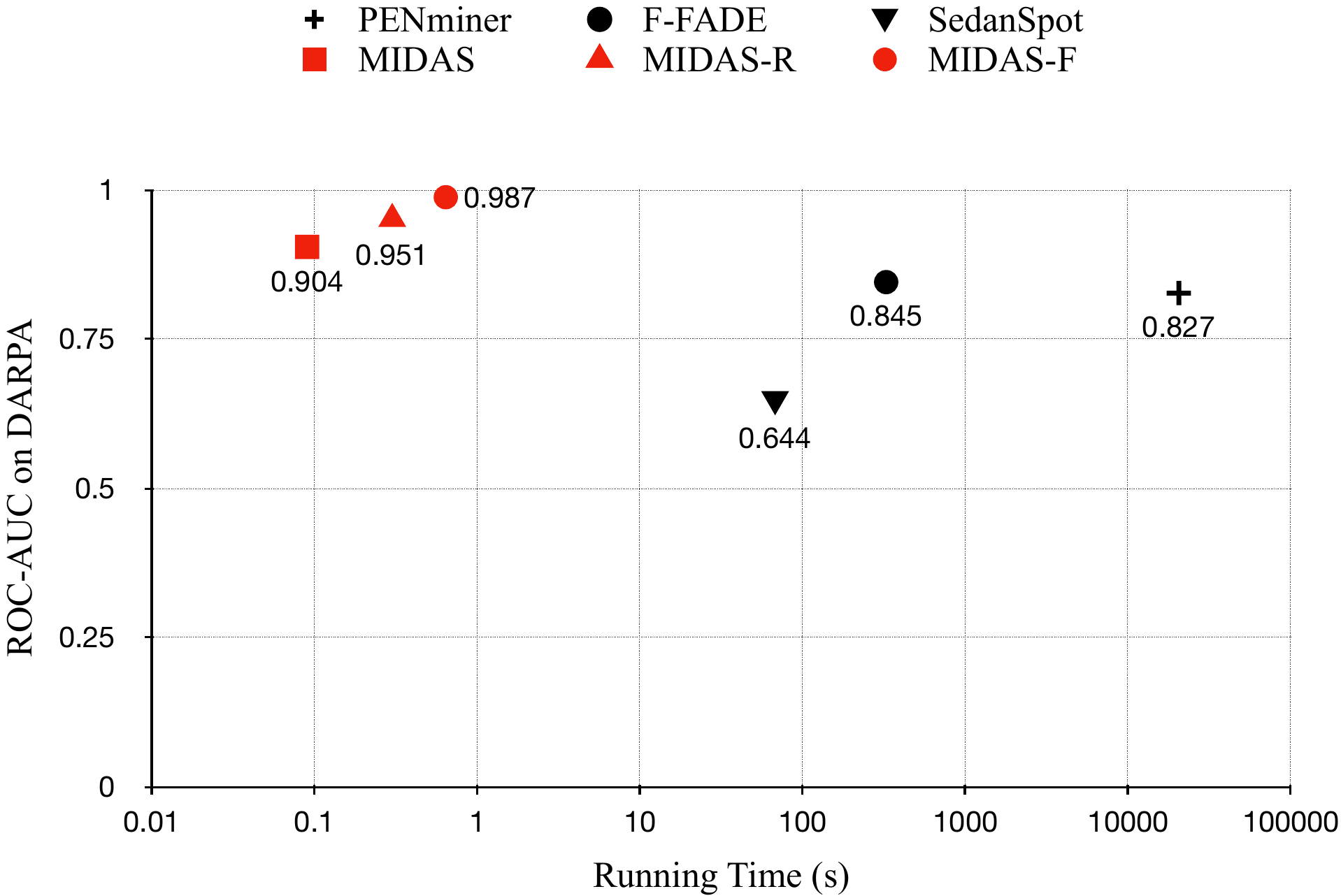}
		\caption{\label{fig:AUCdarpa} ROC-AUC vs. time on \emph{DARPA}}
	\end{figure}

	\begin{figure}[!ht]
		\center
		\includegraphics[width=0.8\columnwidth]{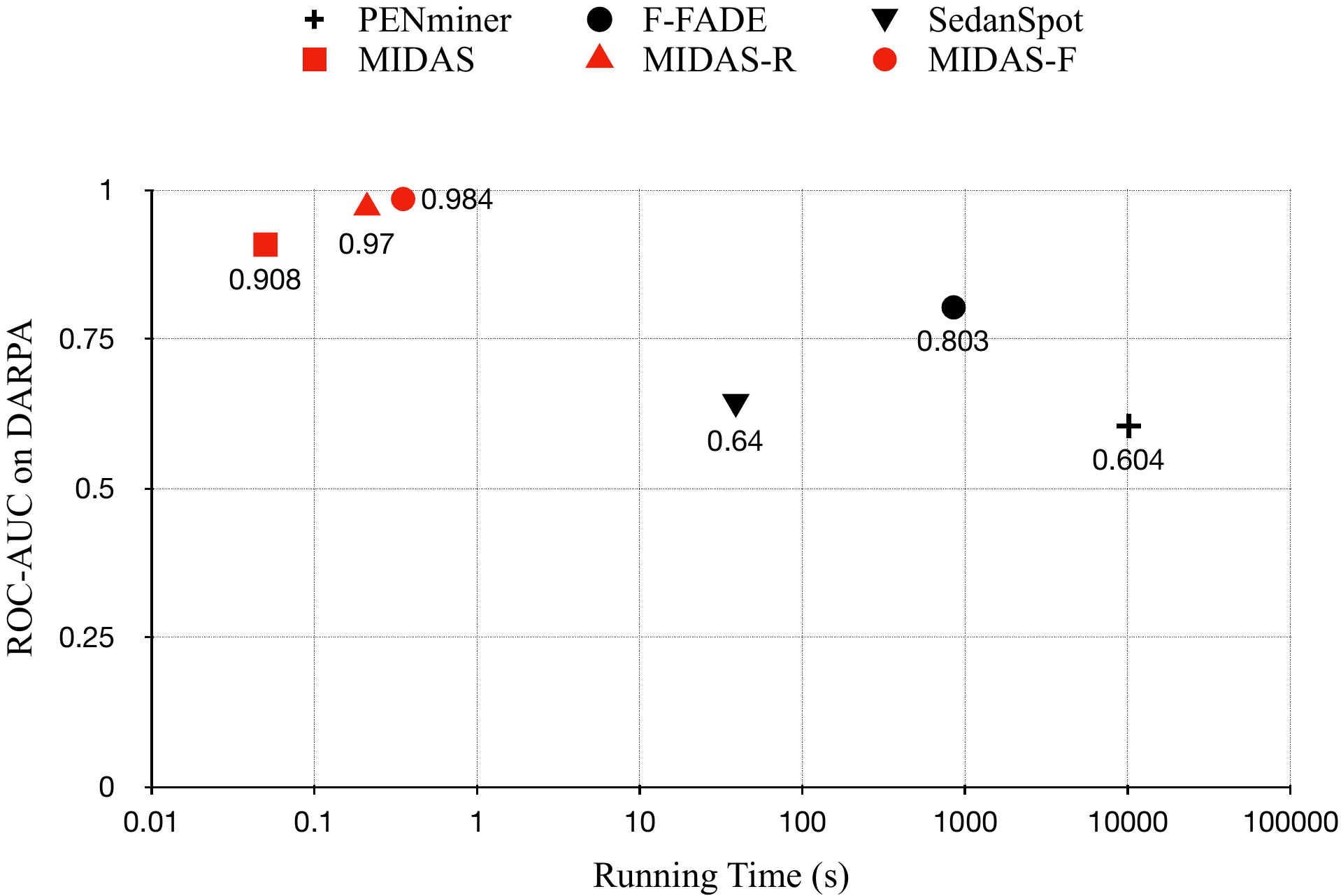}
		\caption{\label{fig:AUCctu} ROC-AUC vs. time on \emph{CTU-13}}
	\end{figure}

	\begin{figure}[!ht]
		\center
		\includegraphics[width=0.8\columnwidth]{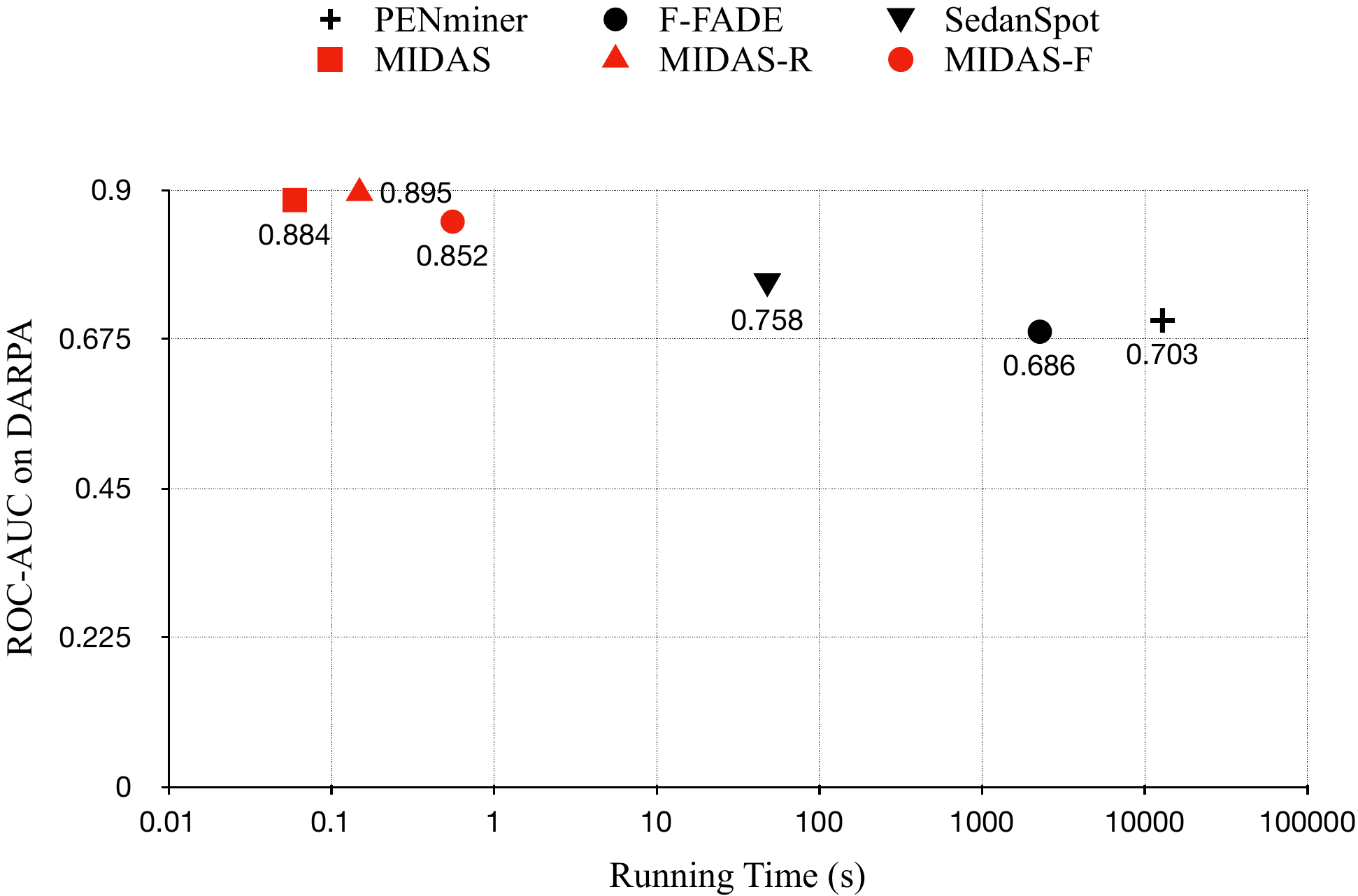}
		\caption{\label{fig:AUCunsw} ROC-AUC vs. time on \emph{UNSW-NB15}}
	\end{figure}

	Table~\ref{tab:FactorVsAUC} shows the influence of the temporal decay factor $\alpha$ on the ROC-AUC for \method-R and \method-F in the \emph{DARPA} dataset. Note that instead of scaling the values in the CMS, \method{} clears (or resets) values in the CMS data structure when the time tick changes; therefore, it is not included. We see that $\alpha=0.9$ gives the maximum ROC-AUC for \method-R ($0.9657$) and $\alpha=0.8$ for \method-F ($0.9876$).

	\begin{table}[!htb]
		\centering
		\caption{Influence of temporal decay factor $\alpha$ on the ROC-AUC in \method-R and \method-F}
		\label{tab:FactorVsAUC}
		\begin{tabular}{@{}lrr@{}}
			\toprule
			$\alpha$ & \method-R & \method-F \\
			\midrule
			$0.1$ & $0.9346$ & $0.9779$ \\
			$0.2$ & $0.9429$ & $0.9801$ \\
			$0.3$ & $0.9449$ & $0.9817$ \\
			$0.4$ & $0.9484$ & $0.9837$ \\
			$0.5$ & $0.9504$ & $0.9852$ \\
			$0.6$ & $0.9526$ & $0.9863$ \\
			$0.7$ & $0.9542$ & $0.9863$ \\
			$0.8$ & $0.9590$ & $0.9883$ \\
			$0.9$ & $0.9657$ & $0.9876$ \\
			\bottomrule
		\end{tabular}
	\end{table}

	Table~\ref{tab:ThresholdVsAUC} shows the influence of the threshold $\theta$ on the ROC-AUC for \method-F in the \emph{DARPA} dataset. If the threshold is too low, even normal edges can be rejected. On the other end, if the threshold is too high ($\theta=10^7$), very few anomalous edges will be rejected, and \method-F (ROC-AUC = $0.9572$) performs similar to \method-R (ROC-AUC = $0.95$). We see that $\theta=10^3$ achieves the maximum ROC-AUC of $0.9853$.

	\begin{table}[!htb]
		\centering
		\caption{Influence of threshold $\theta$ on the ROC-AUC in \method-F}
		\label{tab:ThresholdVsAUC}
		\begin{tabular}{@{}lr@{}}
			\toprule
			$\theta$ & ROC-AUC \\
			\midrule
			$10^0$ & $0.9838$ \\
			$10^1$ & $0.9840$ \\
			$10^2$ & $0.9839$ \\
			$10^3$ & $0.9853$ \\
			$10^4$ & $0.9807$ \\
			$10^5$ & $0.9625$ \\
			$10^6$ & $0.9597$ \\
			$10^7$ & $0.9572$ \\
			\bottomrule
		\end{tabular}
	\end{table}

	Table~\ref{tab:BucketsVsAUC} shows the ROC-AUC vs. number of buckets ($b$) in CMSs on the \emph{UNSW-NB15} dataset.
	We can observe the increase in the performance, which indicates that increasing the buckets helps alleviate the effect of conflicts, and further reduce the false positive rate of the resulting scores.
	Also, note that the ROC-AUC does not change after $10,000$ buckets, one possible reason is that the number of columns is sufficiently high to negate the influence of conflicts.
	This also simulates the ``no-CMS'' situation, i.e., the edge counts are maintained in an array of infinite size.

	\begin{table}[!htb]
		\centering
		\caption{Influence of the number of buckets on the ROC-AUC in \method, \method-R, and \method-F}
		\label{tab:BucketsVsAUC}
		\begin{tabular}{@{}lrrr@{}}
			\toprule
			$b$ & \method & \method-R & \method-F \\
			\midrule
			$10^2$ & $0.7978$ & $0.8161$ & $0.8653$ \\
			$10^3$ & $0.8732$ & $0.8418$ & $0.8863$ \\
			$10^4$ & $0.8842$ & $0.8517$ & $0.8952$ \\
			$10^5$ & $0.8842$ & $0.8517$ & $0.8952$ \\
			$10^6$ & $0.8842$ & $0.8517$ & $0.8952$ \\
			$10^7$ & $0.8842$ & $0.8517$ & $0.8952$ \\
			\bottomrule
		\end{tabular}
	\end{table}

	For \method{}-R and \method{}-F, we also test the effect of summing the three anomaly scores, one for the edge $(u,v)$, one for node $u$, and one for node $v$. The scores are not significantly different: with default parameters, the ROC-AUC is $0.95$ for \method-R (vs. $0.95$ using maximum) and $0.98$ for \method-F (vs. $0.99$ using maximum).

	\subsection{Scalability}

	Table~\ref{tab:times} shows the running time for the baselines and \method\ algorithms.	Compared to \sedanspot, on all the 5 datasets, \method\ speeds up by $623-800\times$, \method-R speeds up by $183-326\times$, and \method-F speeds up by $85-286\times$. Compared to F-FADE, on all the 5 datasets, \method\ speeds up by $806-37782\times$, \method-R speeds up by $366-15112\times$, and \method-F speeds up by $366-4047\times$. Compared to PENminer, on all the 5 datasets, \method\ speeds up by $101419-214282\times$, and \method-R speeds up by $46099-85712\times$, \method-F speeds up by $22958-47324\times$.

	\begin{table*}[!htb]
		\centering
		\caption{Running time for different datasets in seconds}
		\label{tab:times}
		\begin{tabular}{lrrrrrr}
			\toprule
			Dataset & PENminer & F-FADE & \sedanspot & \method & \method-R       & \method-F       \\
			\midrule
			\emph{DARPA}        & $20423$s & $325.1$s & $67.54$s   & $0.09$s & $0.30$s   & $0.64$s   \\
			\emph{CTU-13}       & $10065$s & $844.2$s & $38.73$s   & $0.05$s & $0.21$s   & $0.35$s   \\
			\emph{UNSW-NB15}      & $12857$s & $2267$s  & $48.03$s   & $0.06$s & $0.15$s   & $0.56$s   \\
			\emph{TwitterWorldCup}  & $3786$s  & $141.7$s & $22.92$s   & $0.03$s & $0.07$s   & $0.08$s   \\
			\emph{TwitterSecurity}  & $5071$s  & $40.34$s & $31.18$s   & $0.05$s & $0.11$s   & $0.11$s   \\
			\bottomrule
		\end{tabular}
	\end{table*}

	\sedanspot{} requires several subprocesses (hashing, random-walking, reordering, sampling, etc), resulting in a large computation time. For PENminer and F-FADE, while the python implementation is a factor, the algorithm procedures also negatively affect their running speed. PENminer requires active pattern exploration and F-FADE needs expensive factorization operations.
	For \method, the improvement of running speed is through both, the algorithm procedure as well as the implementation. The algorithm procedure is less complicated than baselines; for each edge, the only operations are updating CMSs (hashing) and computing scores, and both are within constant time complexity. The implementation is well optimized and utilizes techniques like auto-vectorization to boost execution efficiency.

	Figure~\ref{fig:scaling} shows the scalability of \method, \method-R, and \method-F algorithms. We plot the time required to process the first $2^{16}, 2^{17},\ldots,2^{22}$ edges of the \emph{DARPA} dataset. This confirms the linear scalability of \method{} algorithms with respect to the number of edges in the input dynamic graph due to its constant processing time per edge. Note that \method{}, \method{}-R and \method-F can process $4.5M$ edges within $1$ second, allowing real-time anomaly detection.

	\begin{figure}[!htb]
		\center
		\includegraphics[width=0.65\columnwidth]{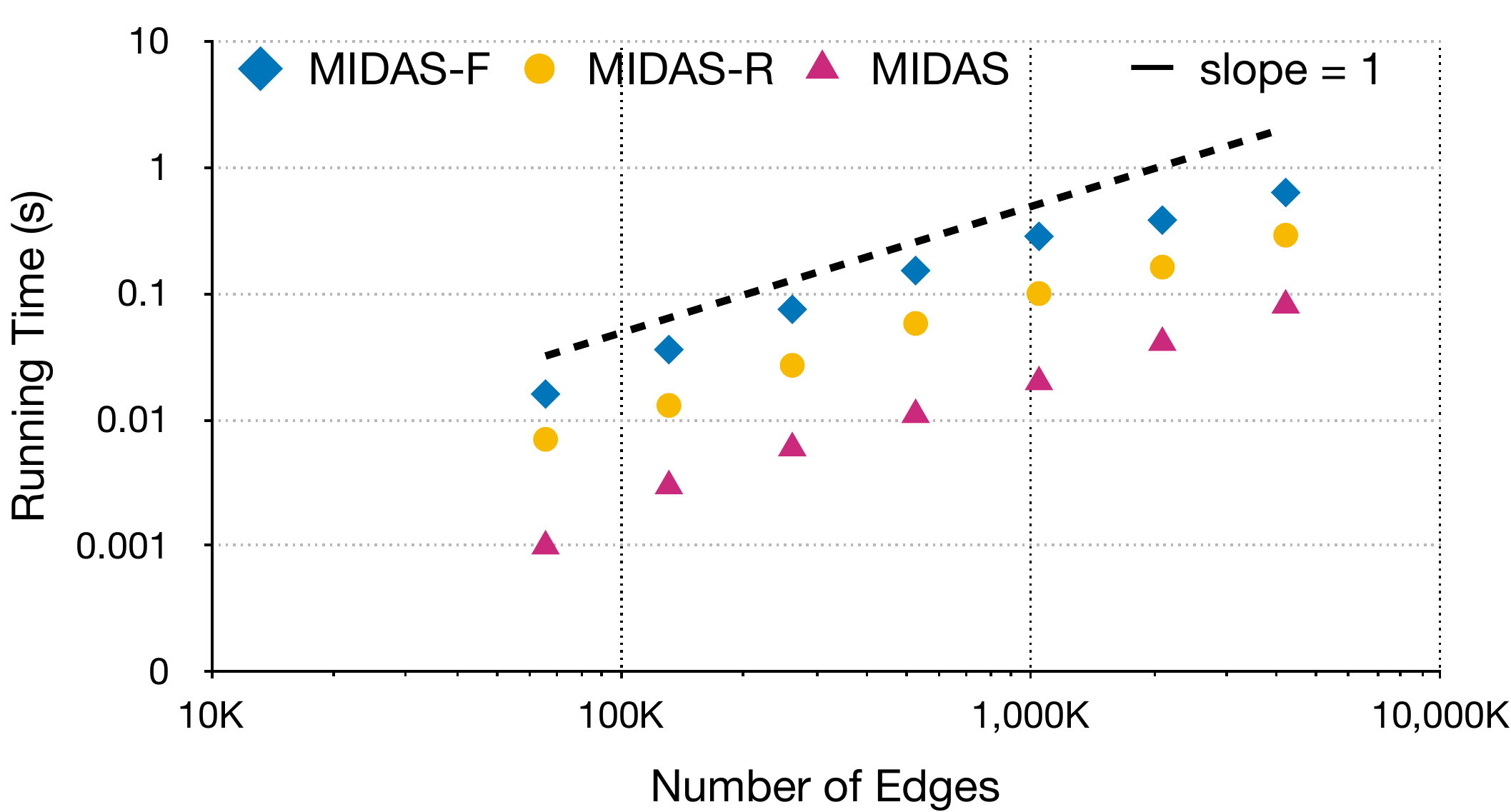}
		\caption{\label{fig:scaling} \method{}, \method-R and \method-F scale linearly with the number of edges in the input dynamic graph.}
	\end{figure}

	Figure~\ref{fig:frequency} plots the number of edges and the time to process each edge in the \emph{DARPA} dataset. Due to the limitation of clock accuracy, it is difficult to obtain the exact time of each edge. But we can approximately divide them into two categories, i.e., less than $1\mu s$ and greater than $1\mu s$. All three methods process majority of the edges within $1\mu s$.

	\begin{figure}[!htb]
		\center
		\includegraphics[width=0.65\columnwidth]{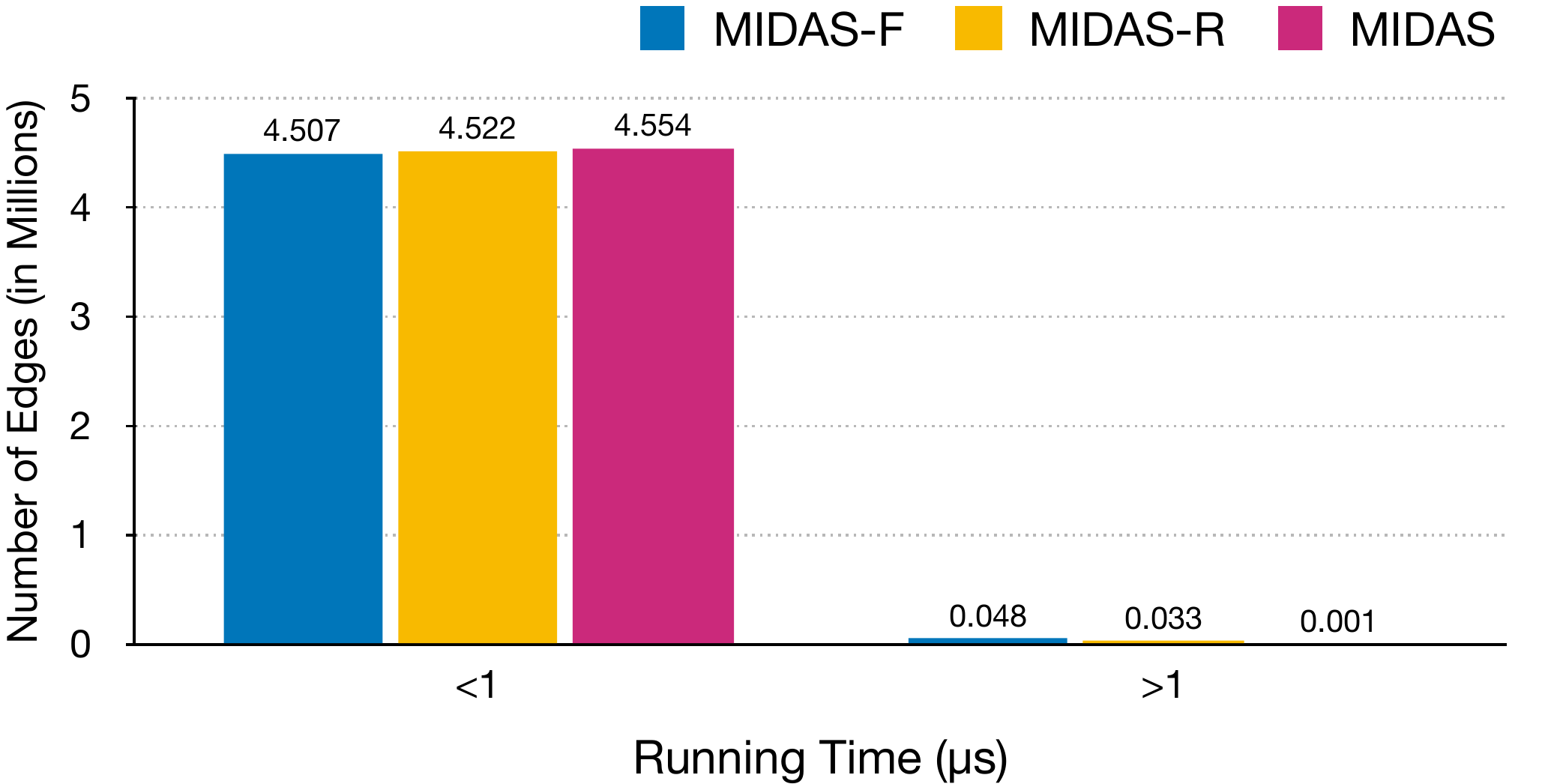}
		\caption{Distribution of processing times for $\sim4.5M$ edges of \emph{DARPA} dataset.}\label{fig:frequency}
	\end{figure}

	Figure~\ref{fig:threshold} shows the dependence of the running time on the threshold for \method-F. We observe that the general pattern is a line with slope close to $0$. Therefore, the time complexity does not depend on the threshold.

	\begin{figure}[!htb]
		\center{\includegraphics[width=0.65\columnwidth]{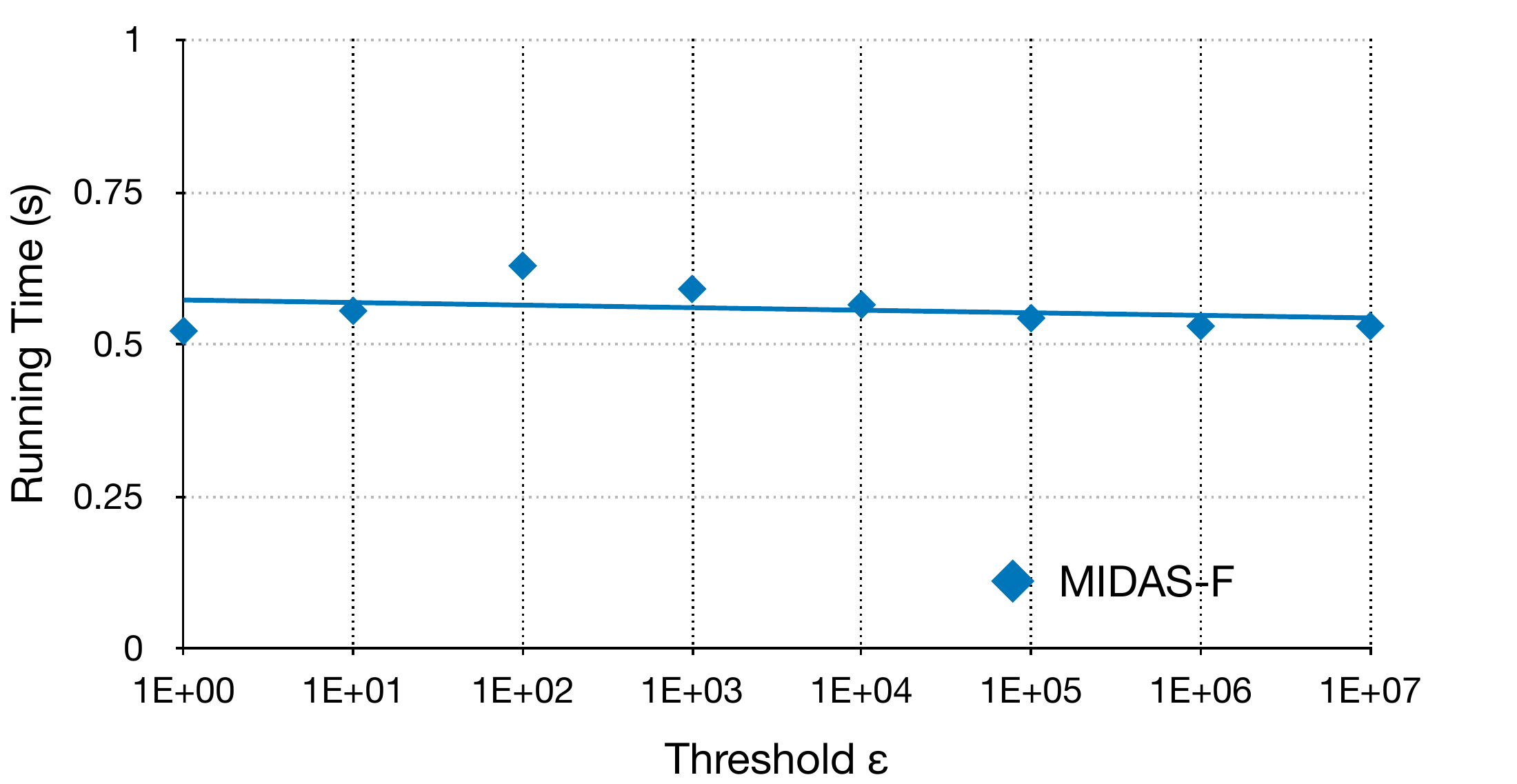}}
		\caption{\label{fig:threshold} Running time of \method-F does not depend on the threshold $\theta$}
	\end{figure}

	Figure~\ref{fig:hashfunctions} shows the dependence of the running time on the number of hash functions and linear scalability.

	\begin{figure}[!htb]
		\center{\includegraphics[width=0.65\columnwidth]{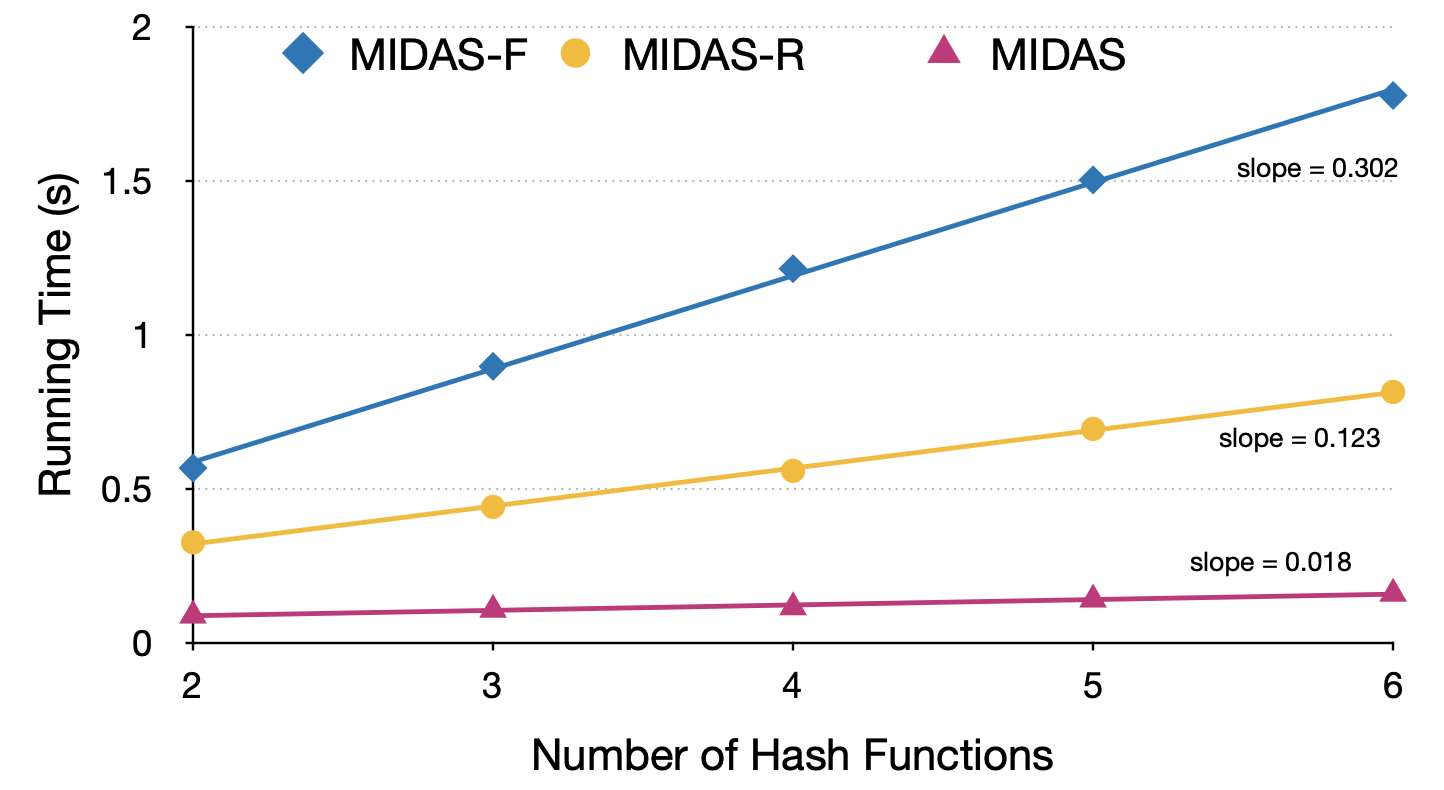}}
		\caption{\label{fig:hashfunctions} \method{}, \method-R and \method-F scale linearly with the number of hash functions.}
	\end{figure}

	Figure~\ref{fig:buckets} shows the dependence of the running time on the number of buckets. In general, the time increases with the number of buckets, but \method-F is more sensitive to the number of buckets. This is because \method-F requires updating the CMS data structure, which, due to the nested selection operation, cannot be vectorized. On the other hand, in \method{} and \method-R, the clearing and $\alpha$ reducing operations can be efficiently vectorized.

	\begin{figure}[!htb]
		\center{\includegraphics[width=0.65\columnwidth]{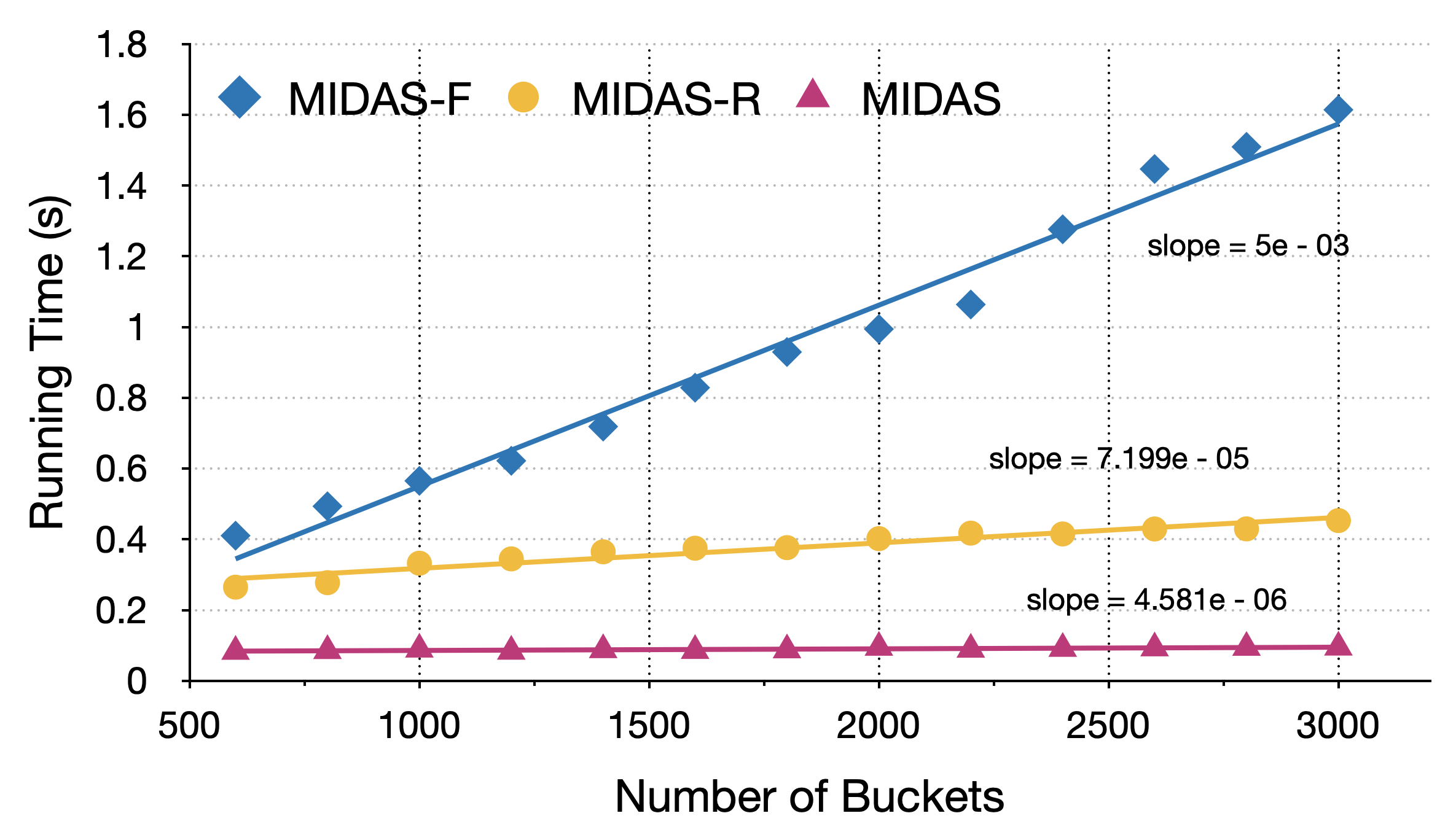}}
		\caption{\label{fig:buckets} \method{}, \method-R and \method-F scale linearly with the number of buckets.}
	\end{figure}

	\subsection{Real-World Effectiveness}

	We measure anomaly scores using \method, \method-R, \method-F, \sedanspot, PENminer, and F-FADE on the \emph{TwitterSecurity} dataset.
	Figure~\ref{fig:security} plots the normalized anomaly scores vs. day (during the four months of 2014).
	We aggregate edges for each day by taking the highest anomaly score.
	Anomalies correspond to major world news such as the Mpeketoni attack (event 6) or the Soma Mine explosion (event 1).

	\begin{figure}[!htb]
		\centering
		\includegraphics[width=0.5\columnwidth]{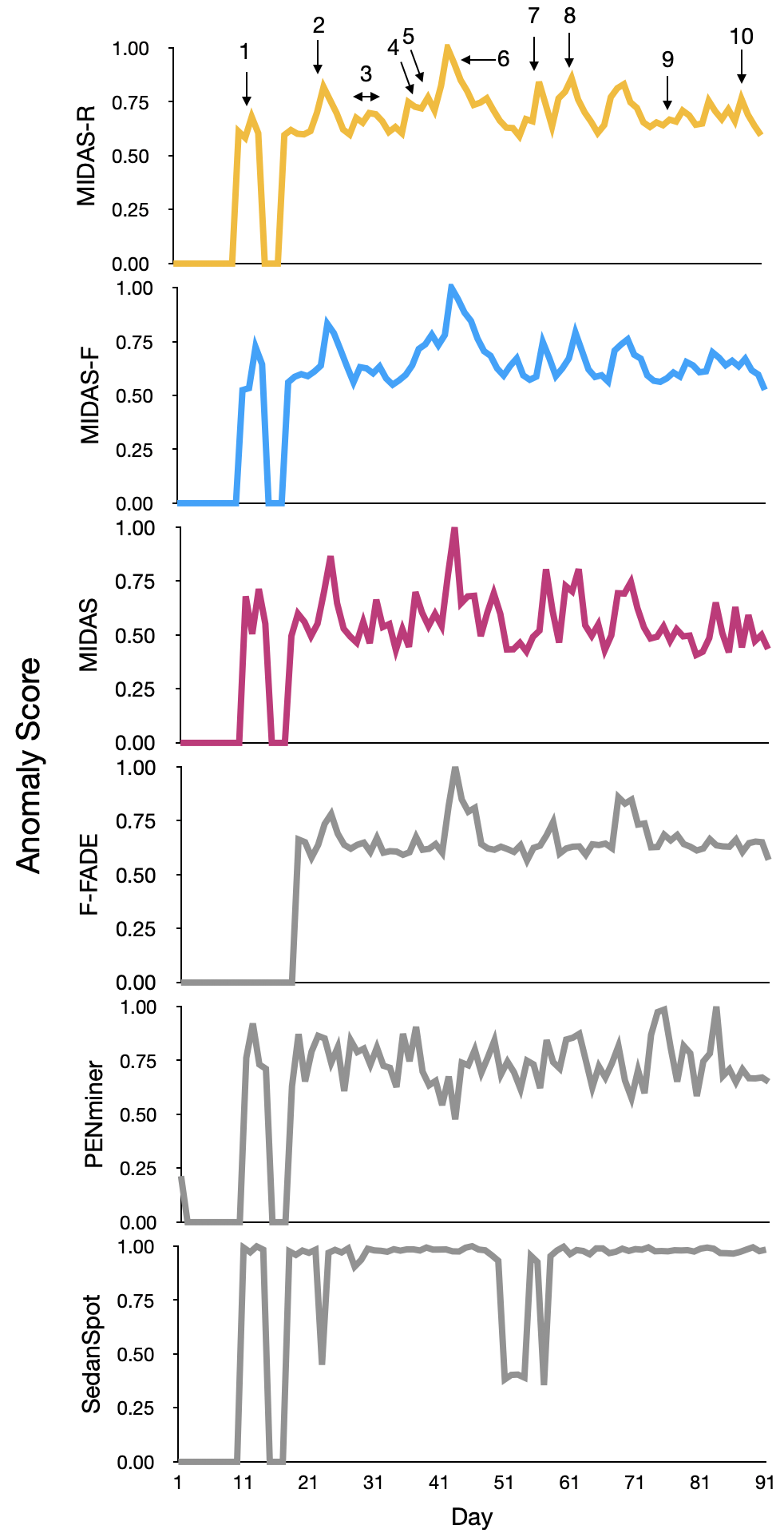}
		\caption{Anomalies detected by \method, \method-R and \method-F correspond to major security-related events in \emph{TwitterSecurity}}\label{fig:security}
	\end{figure}

	\sedanspot\ gives relatively high scores for all days making it difficult to spot anomalies (events). F-FADE produces the highest score near event 6 and peaks at events 2 and 8. However, for other days, scores are maintained around a static level, which provides no useful information in detecting rest events.
	Also note that as F-FADE requires initial learning, thus there are no scores around event 1.
	PENminer's scores keep fluctuating during the four months.
	It would be hard to learn anomalies from the produced scores.
	For \method\ and its variants, we can see four apparent peaks near major events like 2, 6, 7, 8, and at events 1 and 10, small peaks are also noticeable, though less obvious.
	Hence, we can see our proposed algorithm can extract out more anomalous events from real-world social networks compared with baselines.

	The anomalies detected by \method, \method-R and \method-F coincide with the ground events in the \emph{TwitterSecurity} timeline as follows:

	\begin{enumerate}
		\footnotesize
		\item 13-05-2014. Turkey Mine Accident, Hundreds Dead.
		\item 24-05-2014. Raid.
		\item 30-05-2014. Attack/Ambush.\\ 03-06-2014. Suicide bombing.
		\item 09-06-2014. Suicide/Truck bombings.
		\item 10-06-2014. Iraqi Militants Seized Large Regions.\\ 11-06-2014. Kidnapping.
		\item 15-06-2014. Attack.
		\item 26-06-2014. Suicide Bombing/Shootout/Raid.
		\item 03-07-2014. Israel Conflicts with Hamas in Gaza.
		\item 18-07-2014. Airplane with 298 Onboard was Shot Down over Ukraine.
		\item 30-07-2014. Ebola Virus Outbreak.
	\end{enumerate}

	\textbf{Microcluster anomalies:} Figure \ref{fig:micro} corresponds to Event $7$ in the \emph{TwitterSecurity} dataset. Single edges in the plot denote $444$ actual edges, while double edges in the plot denote $888$ actual edges between the nodes. This suddenly arriving (within $1$ day) group of suspiciously similar edges is an example of a microcluster anomaly which \method, \method-R\ and \method-F detect, but \sedanspot{} misses.

	\begin{figure}[H]
		\centering
		\includegraphics[width=0.5\columnwidth]{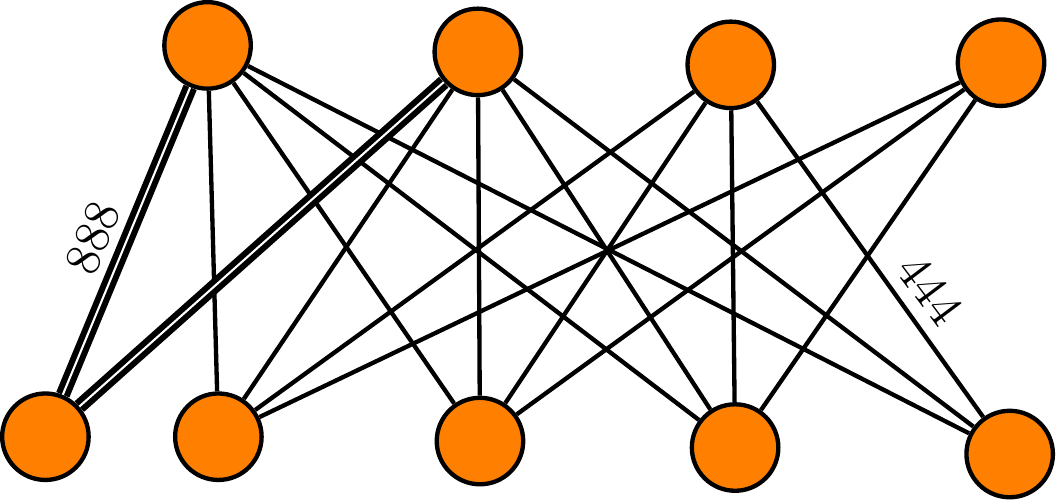}
		\caption{Microcluster Anomaly in \emph{TwitterSecurity}}\label{fig:micro}
	\end{figure}

	\section{Conclusion}
	In this paper, we proposed \method, \method-R, and \method-F for microcluster based detection of anomalies in edge streams. Future work could consider more general types of data, including heterogeneous graphs or tensors.
	Our contributions are as follows:
	\begin{enumerate}
		\item Streaming Microcluster Detection: We propose a novel streaming approach combining statistical (chi-squared test) and algorithmic (count-min sketch) ideas to detect microcluster anomalies, requiring constant time and memory.
		\item Theoretical Guarantees: In Theorem \ref{thm:bound}, we show guarantees on the false positive probability of \method.
		\item Effectiveness: Our experimental results show that \method{} outperforms baseline approaches by up to $62$\% higher ROC-AUC, and processes the data orders-of-magnitude faster than baseline approaches.
		\item Filtering Anomalies: We propose a variant, \method-F, that introduces two modifications that aim to filter away anomalous edges to prevent them from negatively affecting the algorithm's internal data structures.
	\end{enumerate}

	\bibliographystyle{ACM-Reference-Format}
	\bibliography{main}

\end{document}